\documentclass[conference]{IEEEtran}
\IEEEoverridecommandlockouts
\usepackage{microtype}
\usepackage{subfigure}
\usepackage{booktabs} 
\usepackage{lineno}
\usepackage[utf8]{inputenc} 
\usepackage[T1]{fontenc}    
\usepackage{url}            
\usepackage{booktabs}       
\usepackage{nicefrac}       
\usepackage{enumitem}
\usepackage[numbers]{natbib}
\usepackage{wrapfig}
\usepackage{float}    
\usepackage{subcaption} 
\usepackage{amsmath,amssymb,amsfonts}
\usepackage{algorithmic}
\usepackage{graphicx}
\usepackage{textcomp}
\usepackage{xcolor}
\def\BibTeX{{\rm B\kern-.05em{\sc i\kern-.025em b}\kern-.08em
    T\kern-.1667em\lower.7ex\hbox{E}\kern-.125emX}}

\newcommand{\Omegab}{\boldsymbol{\Omega}}

\newcommand{\Lcrw}{\Lc_{\rm{RW}}}

\newcommand{\majmaj}{\textbf{\texttt{maj-maj}}\xspace}
\newcommand{\majmin}{\textbf{\texttt{maj-min}}\xspace}
\newcommand{\minmin}{\textbf{\texttt{min-min}}\xspace}

\newcommand{\Wbar}{\overline{\W}}

\newcommand{\Winf}{\W_\infty}
\newcommand{\Hbinf}{\Hb_\infty}

\usepackage[mathscr]{euscript}
\newcommand{\dt}[1]{\frac{\mathrm{d}{#1}}{\mathrm{d}t}}

\newcommand{\smat}{\Y}
\newcommand{\wbars}{\overline{w}}
\newcommand{\hbars}{\overline{h}}
\newcommand{\Hbar}{\overline{\Hb}}
\newcommand{\Lambdab}{\boldsymbol{\Lambda}}

\newcommand{\smatbar}{\Z}
\newcommand{\Sbar}{\Z}
\newcommand{\sbar}{\z}

\newcommand{\Lb}{\vct{L}}

\usepackage{hyperref,graphicx,amsmath,amsfonts,amssymb,bm,url,breakurl,epsfig,epsf,color,fmtcount,semtrans,caption,multirow,comment, boldline, amsthm}
\usepackage{subcaption}
\usepackage{bm}
\usepackage{enumitem}
\usepackage{amssymb}

\setlist[itemize]{leftmargin=5mm}

\usepackage{hyperref}
\usepackage[utf8]{inputenc} 
\usepackage[T1]{fontenc}    
\usepackage{url}            
\usepackage{booktabs}       
\usepackage{nicefrac}       
\usepackage{microtype}      
\usepackage{dsfont}
\usepackage{xspace}

\newcommand{\Vb}{{\mtx{V}}}

\newcommand{\vct}[1]{\bm{#1}}
\newcommand{\mtx}[1]{\bm{#1}}

\usepackage{mathtools}

\usepackage{titlesec}

\usepackage{tikz}
\usepackage{pgfplots}
\usetikzlibrary{pgfplots.groupplots}

\usepackage{movie15}

\usepackage{caption}
\usepackage[bottom,hang,flushmargin]{footmisc} 

\newcommand{\tsn}[1]{{\left\vert\kern-0.25ex\left\vert\kern-0.25ex\left\vert #1 
    \right\vert\kern-0.25ex\right\vert\kern-0.25ex\right\vert}}

\definecolor{darkred}{RGB}{150,0,0}
\definecolor{darkgreen}{RGB}{0,150,0}
\definecolor{darkblue}{RGB}{0,0,200}
\hypersetup{colorlinks=true, linkcolor=darkred, citecolor=darkgreen, urlcolor=darkblue}

\newcommand{\Rb}{\mathbf{R}}

\newcommand{\diag}[1]{\operatorname{diag}(#1)}

\newcommand{\appropto}{\mathrel{\vcenter{
  \offinterlineskip\halign{\hfil$##$\cr
    \propto\cr\noalign{\kern2pt}\sim\cr\noalign{\kern-2pt}}}}}
    
\newcommand{\cut}[1]{\textcolor{red}{}}
\newcommand{\W}{{\vct{W}}}
\newcommand{\Ab}{{\vct{A}}}

\newcommand{\Yb}{\mathbf{Y}}

\newcommand{\Z}{\vct{Z}}

\newcommand{\Ub}{\vct{U}}

\newcommand{\Hb}{{\vct{H}}}

\newcommand{\Y}{\vct{Y}}

\newcommand{\thetab}{\boldsymbol{\theta}}

\newcommand{\x}{\vct{x}}

\newcommand{\ub}{\vct{u}}

\newcommand{\Bb}{\vct{B}}

\newcommand{\eb}{\vct{e}}

\newcommand{\z}{\vct{z}}

\newcommand{\hb}{\vct{h}}

\newcommand{\Rc}{\mathcal{R}}

\newcommand{\Lc}{\mathcal{L}}

\newcommand{\beq}{\begin{equation}}
\newcommand{\eeq}{\end{equation}}
\newcommand{\bea}{\begin{align}}
\newcommand{\eea}{\end{align}}

\newcommand{\R}{\mathbb{R}}

\newcommand{\nn}{\notag}

  \newcommand{\Sigmab}{\boldsymbol\Sigma}
    
\DeclarePairedDelimiterX{\inp}[2]{\langle}{\rangle}{#1, #2}

\newcommand{\Id}{\mathds{I}}
\newcommand{\ones}{\mathds{1}}
\newcommand{\zeros}{\mathbf{0}}

\theoremstyle{plain}
\newtheorem{theorem}{Theorem}[section]
\newtheorem{proposition}[theorem]{Proposition}

\newtheorem{corollary}[theorem]{Corollary}
\theoremstyle{definition}

\theoremstyle{remark}
\newtheorem{remark}[theorem]{Remark}

\makeatletter
\let\href\undefined          
\newcommand{\href}[2]{#2}    
\makeatother
    
\begin{document}

\title{Why Loss Re-weighting Works If You Stop Early: \\
           Training Dynamics of Unconstrained Features
}


\author{\IEEEauthorblockN{Yize Zhao, Christos Thrampoulidis}
\IEEEauthorblockA{
\textit{The University of British Columbia}\\
Vancouver, Canada \\
zhaoyize@ece.ubc.ca,
cthrampo@ece.ubc.ca}
}

\maketitle

\begin{abstract}
The application of loss reweighting in modern deep learning presents a nuanced picture. While it  fails to alter the terminal learning phase in overparameterized deep neural networks (DNNs) trained on high-dimensional datasets, empirical evidence consistently shows it offers significant benefits early in training. To transparently demonstrate and analyze this phenomenon, we introduce a small-scale model (SSM). This model is specifically designed to abstract the inherent complexities of both the DNN architecture and the input data, while maintaining key information about the structure of imbalance within its spectral components.  On the one hand, the SSM reveals how vanilla empirical risk minimization preferentially learns to distinguish majority classes over minorities early in training, consequently delaying minority learning. In stark contrast, reweighting restores balanced learning dynamics, enabling the simultaneous learning of features associated with both majorities and minorities. 
\end{abstract}


\section{Introduction}
Classification of imbalanced datasets is a pervasive challenge in practical machine learning. Most real-world datasets exhibit classes with varying numbers of examples, where classes with significantly fewer training examples are commonly referred to as minorities. A classical and widely adopted approach to learning from imbalanced data is \emph{reweighting}, which modifies the loss function to effectively increase the contribution of minority class examples during training. Concretely, instead of a learning model $f$ that minimizes vanilla empirical risk $\Lc(f) = \frac{1}{n}\sum_{i\in[n]}\ell\left(f(\x_i), y_i\right)$, over training examples of feature-label pairs $(\x_i,y_i)$, loss reweighting minimizes 
\begin{align}\label{eq:loss weighted}
\Lcrw(f) = \frac{1}{n}\sum\nolimits_{i\in[n]} \omega_{y_i} \cdot \ell(f(\x_i), y_i),
\end{align}
where $\omega_{y_i}\propto 1/\hat{\pi}_{y_i}$ is typically set proportional to the inverse empirical frequency\footnote{In practice, it is common setting $\omega_{y_i}\propto 1/\hat{\pi}_{y_i}^\gamma$ where $\gamma\in(0,1)$ is tuned empirically, e.g.,  \cite{TengyuMa,Menon}.} of the class $y_i\in[k]$ to which example $i\in[n]$ belongs. This approach has historically been popular not only due to its simplicity but also because of its statistical optimality: in the population limit, $\Lcrw$ is known to optimize for balanced accuracy, a metric that weighs class-conditional accuracies equally rather than by their class frequencies (see Appendix \ref{sec:reweight_background} for further background).

However, modern machine learning practice, particularly with large-scale deep neural networks (DNNs) and complex high-dimensional datasets, presents a more nuanced picture regarding the statistical optimality and practical efficacy of $\Lcrw$. On one hand, empirical observations \cite{byrd2019effect} and subsequent theoretical justifications \cite{sagawa2020investigation,kini2021label,xu2021understanding} have demonstrated that $\Lcrw$ often fails to substantially improve the accuracy of minority classes over vanilla $\Lc$ if model training continues for a large number of iterations until convergence. On the other hand, empirical evidence also suggests that $\Lcrw$ can improve performance over vanilla $\Lc$ when model training is subjected to early stopping \cite{byrd2019effect,xu2021understanding}. Complementary to this, emphasizing the early-training benefits of reweighting, it has also been found that it can further boost the performance of alternative modern loss functions specifically proposed for imbalanced classification \cite{TengyuMa,Menon,kini2021label,CDT,li2021autobalance}.

A clear understanding of these early-training benefits of $\Lcrw$ is still lacking:
\begin{center}
\emph{Why does reweighting help early in training?} 
\emph{How does it modify the learning dynamics to favor minorities?}
\end{center}

The inherent complexity of DNNs and of real-world datasets makes it challenging to directly answer these questions. 

Our contribution is to identify and analyze a small-scale model (SSM) that transparently demonstrates the exact impact of reweighting on training dynamics, specifically favoring minority classes.

In Section \ref{sec:large scale}, as a motivating experiment, we compare how confusion matrices evolve when training without and with reweighting on a real dataset. This reveals the strong and rapid impact of reweighting at early training stages; see Fig. \ref{fig:confusion_matrix_evolution}.
In Section \ref{sec:small scale}, we introduce our SSM with the explicit goal of explaining this empirically observed behavior. With necessary abstractions of the data, the DNN, and the loss function, and appropriate interpretations of its spectral components, our SSM effectively replicates the empirically observed DNN behaviors related to reweighting.
Finally, Section \ref{sec:analysis_main} leverages the inherent simplicity of the SSM to yield a theoretical understanding of these phenomena.

\section{Motivating Experiment}\label{sec:large scale}

We empirically illustrate the different learning dynamics induced by vanilla empirical risk minimization (ERM) and loss reweighting in an imbalanced classification setting. We construct a 4-class subset of the MNIST dataset with synthetic imbalance: two classes are designated as majorities with 100 training samples each, and the remaining two as minorities with 10 samples each (imbalance ratio $R=10$). We train a 3-layer convolutional neural network (CNN) with embedding dimension $d=32$ under both standard cross-entropy (CE) and reweighted CE, using the Adam optimizer with learning rate $10^{-3}$ and batch size 64. For reweighting, class weights are set inversely proportional to class frequencies.

\begin{figure*}[t]
     \centering
     \includegraphics[width=0.9\linewidth]{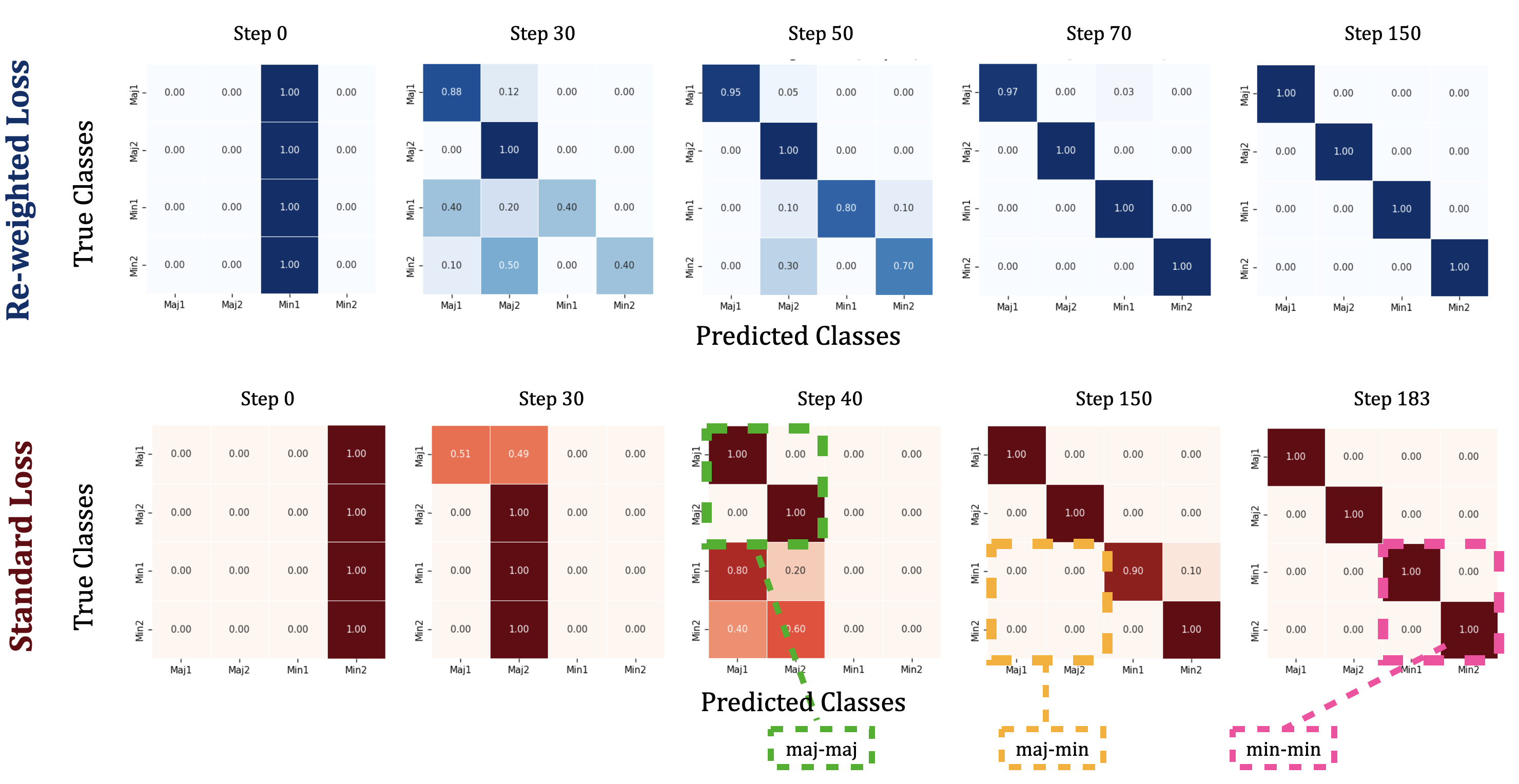}
    \caption{
    \textbf{Confusion matrix evolution during training on imbalanced MNIST} (2 majorities, 2 minorities, imbalance ratio $R=10$).
    Top row: model trained with \textbf{reweighted loss}. 
    Bottom row: model trained with \textbf{standard loss}. 
    Under \textbf{standard loss} (bottom), the model exhibits staged learning dynamics: it first learns to distinguish between majority classes (\textcolor{green}{green boxes}) by step 40, then separates majority from minority classes (\textcolor{orange}{orange boxes}) by step 150, and only later distinguishes between minority classes (\textcolor{magenta}{pink boxes}) at step 183. This progression reflects the ordering of singular values in the simplex-encoded label (SEL) matrix.
    In contrast, under \textbf{reweighted loss} (top), the model learns class distinctions more uniformly and earlier, with balanced improvements across majorities and minorities by step 70. {Our small-scale model replicates this behavior and attributes it to reweighting effectively flattens the spectrum of the labels.}
    }
     \label{fig:confusion_matrix_evolution}
 \end{figure*}

Figure \ref{fig:confusion_matrix_evolution} vividly illustrates the early benefits of loss reweighting. By tracking the confusion matrix on the training data across iterations, we observe the following distinct learning behaviors:

\begin{enumerate}[leftmargin=12pt, itemsep=0pt, parsep=0pt, topsep=0pt, partopsep=0pt]

\item  \textbf{Vanilla ERM exhibits a clear preferential learning order favoring majorities}: The model first learns to correctly classify majority classes, then gradually learns to distinguish minority classes from majority classes, and only in later stages begins to differentiate between various minority classes themselves.

\item \textbf{Reweighting eliminates this preferential ordering, enabling earlier classification of minorities.}
\end{enumerate}

See Appendix \ref{sec:additional_exp} for additional experiments, including the evolution of other training and test metrics (e.g., test-time confusion matrices and balanced accuracy) during training.

\section{Small-Scale Model}\label{sec:small scale}

\subsection{Model Description}

Our model abstracts the inherent complexities of both the DNN architecture and the input data by using a model of unconstrained features (UFM).  
It further combines this with a simplification on the loss from cross-entropy to squared-loss. For concreteness, we focus on a STEP-imbalanced setting. 

 \noindent\textbf{Architecture/data abstraction.}~The UFM substitutes the training embeddings $\hb_{\thetab}(\x_i)$, which are typically high-dimensional representations generated by the last activation layer of a DNN parameterized by $\thetab$, with directly trainable vectors $\hb_i \in \mathbb{R}^d$. This abstraction simultaneously bypasses the network architecture ($\thetab$) and the input data ($\x_i$). The underlying rationale is that overparameterized DNNs are sufficiently expressive to learn data embeddings that are primarily driven by minimizing the training loss, rather than being strictly constrained by the specific architecture. The UFM has been widely adopted and validated, e.g., \cite{yang2017breaking,mixon2020neural,fang2021exploring,zhu2021geometric,zhao2024implicit,garrod2024persistence}.
 \\
Formally, we abstract the DNN as a \textbf{bilinear model} where the logits $\Lb \in \mathbb{R}^{k \times n}$ are defined as $\Lb := \W\Hb$. Here, $\W \in \mathbb{R}^{k \times d}$ represents the network's classifiers (analogous to the last-layer weights), and $\Hb \in \mathbb{R}^{d \times n}$ is a matrix comprising the trainable embeddings $\hb_i, i\in[n]$ for each example. Both $\W$ and $\Hb$ are parameters to be optimized during training. 

 \noindent\textbf{Loss simplification.}~Instead of the cross-entropy loss, we employ the squared-loss. This simplification is a common practice in theoretical analyses, as squared-loss generally leads to more interpretable and analytically tractable learning dynamics. Specifically, we use squared-loss with a \textbf{simplex-encoded label (SEL) matrix} $\smatbar :=(\Id_k-\frac{1}{k}\ones_k\ones_k^\top)\smat$. 
 This matrix is a centered version of the one-hot encoding matrix $\Y \in \{0,1\}^{k \times n}$ of the data. The centering accounts for the property of cross-entropy to yield centered logits, i.e., such that $\mathbf{1}_k^\top\Lb=\mathbf{0}$ \cite{seli}. Viewed this way, the simplification to squared-loss can alternatively be considered a first-order approximation of the cross-entropy loss.

\noindent\textbf{Overparameterization.} To model overparameterization, within our bilinear model, we set hidden dimension $d\geq k$. This ensures model parameters exist that can perfectly interpolate the labels, meaning $\Lb=\Z$.

 \noindent\textbf{STEP-Imbalance.} For concreteness, we assume $k$ classes, where the first $k/2$ are \textbf{majorities} and the remaining $k/2$ are  \textbf{minorities}. Within each group, all classes have an equal number of examples. The number of samples in majorities is $R$ times the number of samples in minorities, where $R$ is the \textbf{imbalance ratio}. For simplicity and without compromising the generality of our findings, we set $k=4$ for our detailed discussion. See Fig.~\ref{fig:4_class_simulation}(a) the 1-hot label matrix $\Y$. 

{\begin{remark}
We view this model as the minimal canonical model capable of transparently justifying how reweighting balances feature learning rates during training. It is minimal because it essentially describes the dynamics of a two-layer linear neural network with standard basis inputs. The model is canonical in that it avoids imposing explicit assumptions on the geometry of the abstracted input data—unlike a simpler linear model would necessitate, e.g., as in \cite{sagawa2020investigation,kini2020analytic,stromberg2025thumb,lai2024sharp,mor2025analytical,behnia2022avoid}. Instead, its abstraction is based on assuming an overparameterized and sufficiently expressive architecture.
\end{remark}}

\subsection{A key Spectral Interpretation}
\begin{figure*}
    \centering
    \includegraphics[width=1\linewidth]{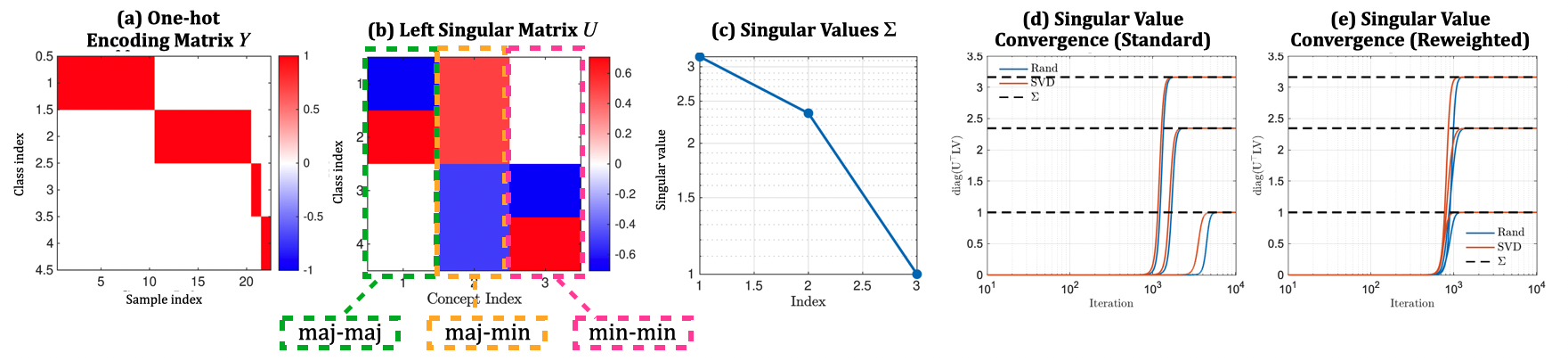}
    \caption{
    \textbf{Spectral analysis of the SSM with 4-class STEP-imbalance (imbalance ratio $R=10$).} 
    (a) One-hot label matrix $\Yb$ for two majority(top) and two minority(bottom) classes. 
    (b) Left singular vectors $\Ub$ of the SEL matrix $\Z = (I - \frac{1}{k}\ones\ones^\top)\Yb$ separate majority-majority (\textcolor{green}{green}), majority-minority (\textcolor{orange}{orange}), and minority-minority (\textcolor{magenta}{pink}) distinctions.
    (c) Singular values show the ordering of semantic importance: majority features dominate.
    (d) Under standard training, features are learned sequentially.
    (e) Reweighting equalizes learning speeds, flattening the spectrum and enabling simultaneous semantics acquisition. {(Blue is random and orange is spectral initialization.)}
    }
    \label{fig:4_class_simulation}
\end{figure*}
The SSM isolates the encoding of the imbalanced data structure within the spectral components of the SEL matrix, $\Z$, in a way that directly reveals how the model prioritizes learning about majorities over minorities.
Let $\Z = \Ub\Sigmab\Vb^\top$ be the SVD of the SEL matrix. It can be formally shown that $\Sigmab$ has $k-1$ non-zero singular values, which are categorized into \emph{three distinct levels} such that the largest and smallest values each have a multiplicity of $k/2-1$, while the middle value has a multiplicity of one. These three distinct levels correspond to three groups of principal components each associated with features that the model must separately learn to distinguish: (i) majority classes from each other (\majmaj feature), (ii) majority classes from minority classes (\majmin feature), and (iii) minority classes from each other (\minmin feature). These features are directly encoded as singular vectors of $\Z$.

 For illustration, consider the case $k=4$. Observe in Fig. \ref{fig:4_class_simulation}(b) the left singular matrix $\Ub$, where each column represents one of these features. We can interpret the sign of each of the $k$ entries in a singular vector $\ub_j$ as an indicator of how strongly the respective class is associated with that principal component: if $\ub_j[c] > 0$ or $\ub_j[c] < 0$, it signifies that class c aligns with or opposes this component, respectively; if $\ub_j[c] = 0$, the class is neutral to that component. 
  We thus find \textbf{three principal components, corresponding to distinct features, ordered by the magnitude of their respective singular values:}
\begin{enumerate}[itemsep=0pt, parsep=0pt, topsep=0pt, partopsep=0pt]
\item $\majmaj$: The first principal component ($j=1$) is neutral to minorities and distinguishes between the two majorities (indicated by the opposing signs of $\ub_1[1]$ and $\ub_1[2]$).
\item $\majmin$: The second principal component ($j=2$) distinguishes between majorities and the minorities (indicated by the opposing signs of $\ub_2[1],\ub_2[2]$ and $\ub_2[3],\ub_2[4]$).
\item $\minmin$: The last principal component ($j=3$) is neutral to majorities and distinguishes between the two minorities (indicated by the opposing signs of $\ub_3[3]$ and $\ub_3[4]$).
\end{enumerate}

\subsection{Reweighting Changes the Learning Dynamics Eliminating Preferential Learning Against Minorities}

We now leverage our SSM to illuminate how reweighting fundamentally alters the learning dynamics. Specifically, we compare the training dynamics under gradient descent for vanilla ERM versus loss reweighting:
\[
\begin{aligned}
\Lc(\W,\Hb)
&:= \frac{1}{n}\sum\nolimits_{i\in[n]}\|\z_i-\W\hb_i\|^2, \textit{versus}
\\
\Lcrw(\W,\Hb)
&:= \frac{1}{n}\sum\nolimits_{i\in[n]}\omega_{y_i}\,\|\z_i-\W\hb_i\|^2 .
\end{aligned}
\]
where the weights $\omega_{y_i}$ are set inversely proportional to the square root of the frequency of the respective class $y_i\in[k]$. Concretely, if the class $y_i$ of the $i$-th example is a  majority, then $\omega_{y_i}\propto\sqrt{(R+1)k/(2R)}$, otherwise, for a minority class, $\omega_{y_i}\propto\sqrt{(R+1)k/2}$.\footnote{

{Note that this is equivalent to using weights $\omega_{y_i}\propto 1/\hat{\pi}_{y_i}^\gamma$ with  $\gamma=0.5$. This particular choice of $\gamma$ is motivated by our theoretical analysis in Section \ref{sec:analysis_main} and, interestingly, aligns with empirical findings from prior work \cite{TengyuMa,Menon,behnia2023implicit,kini2021label,CDT} who report $\gamma=1/2$ as a robust choice in practice.}} For training, we use gradient descent with small random initialization and small step-size.

Fig. \ref{fig:4_class_simulation}(d) and \ref{fig:4_class_simulation}(e) tracks the rates at which the singular values of the logits $\Lb_t$ reach their terminal values in both scenarios (vanilla ERM vs. reweighting) across training iterations $t$. We observe the following about vanilla ERM versus reweighting:
\begin{enumerate}[itemsep=0pt, parsep=0pt, topsep=0pt, partopsep=0pt]
\item \textbf{Terminal Phase Equivalence:} Due to overparameterization ($d\geq k$), there is \emph{no} difference in the terminal phase. In both cases, the singular values of $\Lb_t$ converge to the singular values of the SEL matrix $\Z$. 

\item \textbf{Preferential learning in vanilla ERM hurts minorities:} With vanilla ERM ($\Lc$), the model learns the three singular values in a highly preferential order, with those corresponding to the larger singular values of $\Z$ being learned first.

\item \textbf{Reweighting restores balanced Learning:} In contrast, loss reweighting ($\Lcrw$) eliminates this preferential ordering, leading to all singular values being learned at approximately equal rates.
\end{enumerate}

Recalling that singular values directly correspond to distinct  features,  our SSM thus captures the \textbf{early training benefit of reweighting:} While vanilla ERM $\Lc$ learns \majmaj, \majmin, \minmin features in that specific order (thereby favoring majority learning and delaying minority learning until the end of training), reweighted ERM $\Lcrw$ learns all features at the same rate, thus enabling accurate classification of minorities much earlier. 
Note the direct one-to-one correspondence of these behaviors to what we observed empirically in the DNN experiment presented in Section \ref{sec:large scale}.
\section{Analysis}\label{sec:analysis_main}
We now derive the gradient flow dynamics of the Small-Scale Model (SSM) to theoretically ground the observations in Section \ref{sec:small scale}.
Our analysis proceeds in three steps: (A) We identify a connection between the square-loss UFM and the framework of \citet{saxe2013exact}, establishing closed-form dynamics for vanilla ERM ($\Lc$) under spectral initialization; (B) Leveraging closed-form expressions for the spectral factors of the SEL matrix, we derive analogous closed-form dynamics for reweighted ERM ($\Lcrw$) in Theorem \ref{thm:ours_main}, proving that reweighting effectively flattens the spectrum of the label matrix; (C) Finally, we contrast the relative rates of learning of features, demonstrating that reweighting compresses the effective learning window independent of the imbalance ratio.
\subsection{Dynamics of Vanilla ERM}
Consider the minimization of the UFM with squared loss $\Lc(\W,\Hb)=\frac{1}{2}\|(\Sbar-\W\Hb)\|^2$
This objective fits logits $\Lb = \W\Hb$ to the centered sparsity matrix $\Sbar$ and is mathematically equivalent to training a two-layer linear network with orthogonal inputs. We can therefore adapt the solutions provided by \citet{saxe2013exact} and \citet{bach_saxe}.

Recall $\Sbar = \Ub\Sigmab\Vb^\top$ is the SVD of the SEL matrix. Assume spectral initialization where $\W(0) = e^{-\delta} \Ub \Rb^\top$ and $\Hb(0) = e^{-\delta} \Rb \Vb^\top$ for a small scale $e^{-\delta}$ and arbitrary rotation $\Rb$. As shown in Proposition \eqref{thm:saxe} (see Appendix \ref{app:saxe_connections}), the gradient flow dynamics decouple along the singular directions of $\Sbar$. Specifically, the parameters evolve as:
\begin{align}
    \W(t) = \Ub\sqrt{\Sigmab}\sqrt{\Ab(t)}\Rb^\top, \quad \Hb(t) = \Rb\sqrt{\Sigmab}\sqrt{\Ab(t)}\Vb^\top,
\end{align}
where $\Ab(t) = \text{diag}(a_1(t), \dots, a_r(t))$ describes the scalar evolution of each mode. The factor $a_i(t)$ follows a sigmoid trajectory depending on the singular value $\sigma_i$:
\begin{align}\label{eq:vanilla_dynamics_main}
    a_i(t) = \frac{1}{1 + (\sigma_i e^{2\delta} - 1)e^{-2\sigma_i t}}.
\end{align}
Eq. \eqref{eq:vanilla_dynamics_main} implies that the time $T_i$ required to learn the $i$-th feature (i.e., for $a_i(t) \to 1$) is inversely proportional to its singular value: $T_i \propto {1}/{\sigma_i}$. 
In the STEP-imbalanced setting, the spectrum $\Sigmab$ is skewed. Majority-associated singular values are roughly $\sqrt{R}$ times larger than minority ones. Consequently, vanilla ERM learns majority features at time $T_{\text{maj}} \propto 1/\sqrt{R}$ and minority features at $T_{\text{min}} \propto 1$. The gap between these times increases with the imbalance ratio $R$, explaining the staged learning observed in Fig. \ref{fig:confusion_matrix_evolution} (bottom).

\subsection{Dynamics of Reweighted ERM}
We now analyze the dynamics under the reweighted objective. Consider the weighted square-loss function $\Lcrw(\W,\Hb)$ with class-dependent weights $\omega_i$. In matrix form, $\Lcrw(\W,\Hb)=\frac{1}{2}\|(\Sbar-\W\Hb)\Omegab^{1/2}\|_F^2$, where $\Omegab=\text{diag}(\omega_1, \dots, \omega_n)$.

In the STEP-imbalanced setting with $R$ imbalance ratio, without loss of generality, we assume the first $k/2$ classes are majorities and that minorities have 1 example each; thus, the total number of examples is $n=Rk/2+k/2=(R+1)k/2$. the right singular vectors $\Vb$ take the following closed form:
\begin{align}\label{eq:Vb_def_main}
\Vb^\top = \left[\begin{array}{ccc}
\sqrt\frac{1}{R}\mathbb{F}^\top\otimes\ones_R^\top & \mathbf{0} 
\\-\sqrt{\frac{2}{(R+1)k}}\ones_{Rk/2}^\top  & \sqrt{\frac{2}{(R+1)k}}\ones_{k/2}^\top \\ 
\mathbf{0} & \mathbb{F}^\top
\end{array}\right],
\end{align}
where $\mathbb{F} \in \mathbb{R}^{k/2 \times (k/2-1)}$ is an orthonormal basis for the subspace orthogonal to $\ones_{k/2}$ \citep{seli}. 

We set the weights inversely proportional to the square root of the respective class frequency. That is, the weight for majority samples is $\sqrt{n/R} = \sqrt{(R+1)k/(2R)}$, while for the minorities it is $\sqrt{R}$ times larger. Dropping the constant factor $\sqrt{k/2}$ for simplicity, we choose:
\begin{align}\label{eq:weights STEP_main}
\Omegab:=\sqrt{R+1}\cdot\diag{\begin{bmatrix} \sqrt\frac{1}{R}\,\ones_{R k/2}^\top & \ones_{k/2}^\top \end{bmatrix}}
\,.
\end{align}
\begin{theorem}\label{thm:ours_main} Consider gradient flow (GF) dynamics for minimizing the \textbf{weighted} square-loss UFM 
with the weight matrix in Eq. \eqref{eq:weights STEP_main} under an $R$-STEP-imbalanced setting. Assume spectral initialization:
$
\W(0)=e^{-\delta}\Ub\Rb^\top$ and $\Hb(0)=e^{-\delta}\Rb\Vb^\top$, for partial orthogonal matrix $\Rb\in\R^{d\times (k-1)}$, initialization scale $e^{-\delta}$, and SVD $\Sbar=\Ub\Sigmab\Vb^\top$.
Finally, let $\Lambdab = \text{diag}(\lambda_1, \dots, \lambda_{k-1})$ be the \textbf{effective weight matrix}, where the diagonal entries $\lambda_i$ are defined as:
\begin{align}\label{eq:lambdab def2_main}
    \lambda_i:=\begin{cases}
        \sqrt{\frac{R+1}{R}} & i\in[k/2-1]
        \\
        \frac{\sqrt{R}+1}{\sqrt{R+1}} & i=k/2
        \\
        \sqrt{R+1} & i=k/2+1,\ldots,k-1
    \end{cases}
\end{align}
Then the iterates $\W(t),\Hb(t)$ of GF evolve as follows:
\begin{align}
    \W(t) = \Ub\sqrt{\Sigmab}\sqrt{\Bb(t)}\Rb^\top \quad\text{and} \quad \Hb(t) = \Rb\sqrt{\Sigmab}\sqrt{\Bb(t)}\Vb^\top
\end{align}
for $\Bb(t)=\diag{\beta_1(t),\ldots,\beta_{k-1}(t)}$ with
\begin{align}\label{eq:a_i(t) exact_main}
    \beta_i(t)=\frac{1}{ 1 + ( \sigma_i e^{2\delta} -1) e^{-2\sigma_i \lambda_i t}}
    , ~i\in[k-1].
\end{align}
Moreover, the time-rescaled factors $\beta_i(\delta t)$ converge to a step function as $\delta\rightarrow\infty$:
\begin{align}
    \beta_i(\delta t)\rightarrow \frac{1}{1+\sigma_i}\ones[{t=T_i}]+ \ones[{t>T_i}],
\end{align}
where $T_i=1/(\sigma_i\lambda_i)$. Thus, the $i$-th component is learned at a time inversely proportional to $\lambda_i\cdot\sigma_i$.
\end{theorem}
\begin{proof}
 The GF updates are given by:
\begin{subequations}\label{eq:original weight_main}
\begin{align}
    \dt{\W(t)} &= -\left(\smatbar\Omegab^{1/2}-\W(t)\Hb(t)\Omegab^{1/2}\right)\Omegab^{1/2}\Hb(t)^\top \nonumber \\
    &= -\left(\smatbar-\W(t)\Hb(t)\right)\Omegab\Hb(t)^\top
    \\
        \dt{\Hb(t)} &= -\W(t)^\top\left(\smatbar\Omegab^{1/2}-\W(t)\Hb(t)\Omegab^{1/2}\right) \Omegab^{1/2} \nonumber \\
        &= -\W(t)^\top\left(\smatbar-\W(t)\Hb(t)\right) \Omegab
\end{align}
\end{subequations}
where $\W(0)=e^{-\delta}\Ub\Rb^\top$ and $\Hb(0)=e^{-\delta}\Rb\Vb^\top$. As in \citet{saxe2013exact,saxe2019mathematical}, we change variables to $\Wbar(t), \Hbar(t)$ defined as
\[
\W(t)=\Ub\Wbar(t)\Rb^\top \qquad\text{and}\qquad\Hb(t)=\Rb\Hbar(t)\Vb^\top\,.
\]
At $t=0$, these matrices are diagonal and equal to $e^{-\delta}\Id$ by the initialization assumption. Substituting these into the update equations \eqref{eq:original weight_main} and using $\smatbar=\Ub\Sigmab\Vb^\top$, we obtain:
\begin{subequations}\label{eq:original weight 2_main}
\begin{align}
    \Ub\dt{\Wbar(t)}\Rb^\top &=
    \Ub\left(\Sigmab-\Wbar(t)\Hbar(t)\right)\Vb^\top\Omegab\Vb\Hbar(t)^\top\Rb^\top
    \\
        \Rb\dt{\Hbar(t)}\Vb^\top &= \Rb\Wbar(t)^\top\left(\Sigmab-\Wbar(t)\Hbar(t)\right)\Vb^\top \Omegab\,.
\end{align}
\end{subequations}
Left and right multiplying these by the partial unitary matrices $\Ub,\Rb,\Vb$ yields the term $\Vb^\top\Omegab\Vb$. Using the closed-form definition of $\Vb$ in Eq. \eqref{eq:Vb_def_main} and weights $\Omegab$ in Eq. \eqref{eq:weights STEP_main}, a direct algebraic calculation yields $\Vb^\top\Omegab\Vb = \Lambdab$ defined in Eq. \eqref{eq:lambdab def2_main}. Thus, we arrive at:
\begin{subequations}\label{eq:original weight 3_main}
\begin{align}
    \dt{\Wbar(t)} &= -
    \left(\Sigmab-\Wbar(t)\Hbar(t)\right)\Lambdab\Hbar(t)^\top
    \\
        \dt{\Hbar(t)}&= - \Wbar(t)^\top\left(\Sigmab-\Wbar(t)\Hbar(t)\right)\Lambdab\,.
\end{align}
\end{subequations}
Since $\Lambdab$ and $\Sigmab$ are diagonal, and $\Wbar,\Hbar$ are initialized as diagonal matrices, they remain diagonal throughout training. This results in a decoupled system of differential equations identical in form to the unweighted case in \citet{saxe2013exact}, but with effective rates scaled by $\Lambdab$. The remainder of the proof follows analogously and is omitted due to space limits.
\end{proof}

\subsection{Implications for Learning Times}
The key implication of Theorem \ref{thm:ours_main} is that reweighting modifies the learning rate for each feature $i$ by a factor of $\lambda_i$. A critical metric for comparison is the \emph{effective learning window} $\Delta T$, defined as the relative time difference between learning the first and last principal components:
\begin{align}
    \Delta T := \frac{T_{\max} - T_{\min}}{T_{\min}}.
\end{align}

For vanishing initialization, we find that under vanilla ERM, the window scales with the imbalance ratio $R$ (see Cor. \eqref{cor:delta}):
\begin{align}
    \Delta T_{\rm{vanilla}} = \sqrt{R}-1.
\end{align}
This value grows unboundedly with $R$, explaining the severe delay in minority learning. In contrast, for reweighted ERM, the effective weights $\lambda_i$ counteract the spectral decay. Specifically, the singular values of the effective sparsity matrix become equal. The window becomes:
\begin{align}
    \Delta T_{\rm{RW}} = \sqrt{2}\frac{\sqrt{R+1}}{\sqrt{R}+1} - 1.
\end{align}
This value is upper bounded by $\sqrt{2}-1$ irrespective of the imbalance $R$, confirming that reweighting effectively compresses the learning timeline and enables simultaneous feature acquisition. See Appendix \ref{app:saxe_connections} for full analysis and further experimental validation.

\section{Future Work}
Our model already rather effectively captures the early-training impacts of reweighting under overparameterization. In future work, we aim to address limitations by extending the analysis beyond squared loss and also to test-time dynamics.

\section{Acknowledgments}
This work was funded by the NSERC Discovery Grant No. 2021-03677, the Alliance Grant ALLRP 581098-22, and an Alliance Mission Grant. The authors also acknowledge use of the Sockeye cluster by UBC Advanced Research Computing.

\bibliography{refs,transformers,refs_NC,bib_extra,compbib}

@article{yang2017breaking,
  title={Breaking the softmax bottleneck: A high-rank RNN language model},
  author={Yang, Zhilin and Dai, Zihang and Salakhutdinov, Ruslan and Cohen, William W},
  journal={arXiv preprint arXiv:1711.03953},
  year={2017}
}

@article{kini2021label,
  title={Label-imbalanced and group-sensitive classification under overparameterization},
  author={Kini, Ganesh Ramachandra and Paraskevas, Orestis and Oymak, Samet and Thrampoulidis, Christos},
  journal={Advances in Neural Information Processing Systems},
  volume={34},
  pages={18970--18983},
  year={2021}
}

@article{kini2020analytic,
	Author = {Kini, Ganesh and Thrampoulidis, Christos},
	Journal = {arXiv preprint arXiv:2001.11572},
	Title = {Analytic Study of Double Descent in Binary Classification: The Impact of Loss},
	Year = {2020}}

@INPROCEEDINGS{demirkaya2020exploring,
  author={Demirkaya, Ahmet and Chen, Jiasi and Oymak, Samet},
  booktitle={2020 54th Annual Conference on Information Sciences and Systems (CISS)}, 
  title={Exploring the Role of Loss Functions in Multiclass Classification}, 
  year={2020},
  volume={},
  number={},
  pages={1-5},
  doi={10.1109/CISS48834.2020.1570627167}}

@article{hui2020evaluation,
  title={Evaluation of neural architectures trained with square loss vs cross-entropy in classification tasks},
  author={Hui, Like and Belkin, Mikhail},
  journal={arXiv preprint arXiv:2006.07322},
  year={2020}
}

@article{stromberg2025thumb,
  title={Thumb on the Scale: Optimal Loss Weighting in Last Layer Retraining},
  author={Stromberg, Nathan and Thrampoulidis, Christos and Sankar, Lalitha},
  journal={arXiv preprint arXiv:2506.20025},
  year={2025}
}

@inproceedings{behnia2022avoid,
  title={On how to avoid exacerbating spurious correlations when models are overparameterized},
  author={Behnia, Tina and Wang, Ke and Thrampoulidis, Christos},
  booktitle={2022 IEEE International Symposium on Information Theory (ISIT)},
  pages={121--126},
  year={2022},
  organization={IEEE}
}

@article{wang2021importance,
  title={Is importance weighting incompatible with interpolating classifiers?},
  author={Wang, Ke Alexander and Chatterji, Niladri S and Haque, Saminul and Hashimoto, Tatsunori},
  journal={arXiv preprint arXiv:2112.12986},
  year={2021}
}

@inproceedings{lai2024sharp,
  title={Sharp Analysis of Out-of-Distribution Error for “Importance-Weighted” Estimators in the Overparameterized Regime},
  author={Lai, Kuo-Wei and Muthukumar, Vidya},
  booktitle={2024 IEEE International Symposium on Information Theory (ISIT)},
  pages={3701--3706},
  year={2024},
  organization={IEEE}
}

@article{wang2023unified,
  title={A unified generalization analysis of re-weighting and logit-adjustment for imbalanced learning},
  author={Wang, Zitai and Xu, Qianqian and Yang, Zhiyong and He, Yuan and Cao, Xiaochun and Huang, Qingming},
  journal={Advances in Neural Information Processing Systems},
  volume={36},
  pages={48417--48430},
  year={2023}
}

@inproceedings{welfert2024theoretical,
  title={Theoretical guarantees of data augmented last layer retraining methods},
  author={Welfert, Monica and Stromberg, Nathan and Sankar, Lalitha},
  booktitle={2024 IEEE International Symposium on Information Theory (ISIT)},
  pages={581--586},
  year={2024},
  organization={IEEE}
}

@article{zhai2022understanding,
  title={Understanding why generalized reweighting does not improve over erm},
  author={Zhai, Runtian and Dan, Chen and Kolter, Zico and Ravikumar, Pradeep},
  journal={arXiv preprint arXiv:2201.12293},
  year={2022}
}

@article{mor2025analytical,
  title={An Analytical Model for Overparameterized Learning Under Class Imbalance},
  author={Mor, Eliav and Carmon, Yair},
  journal={arXiv preprint arXiv:2503.05289},
  year={2025}
}

@article{xu2021understanding,
  title={Understanding the role of importance weighting for deep learning},
  author={Xu, Da and Ye, Yuting and Ruan, Chuanwei},
  journal={arXiv preprint arXiv:2103.15209},
  year={2021}
}

@article{garrod2024persistence,
  title={The Persistence of Neural Collapse Despite Low-Rank Bias: An Analytic Perspective Through Unconstrained Features},
  author={Garrod, Connall and Keating, Jonathan P},
  journal={arXiv preprint arXiv:2410.23169},
  year={2024}
}

@article{sukenik2023deep,
  title={Deep neural collapse is provably optimal for the deep unconstrained features model},
  author={S{\'u}ken{\'\i}k, Peter and Mondelli, Marco and Lampert, Christoph H},
  journal={Advances in Neural Information Processing Systems},
  volume={36},
  pages={52991--53024},
  year={2023}
}

@article{zhao2024implicit,
  title={Implicit Geometry of Next-token Prediction: From Language Sparsity Patterns to Model Representations},
  author={Zhao, Yize and Behnia, Tina and Vakilian, Vala and Thrampoulidis, Christos},
  journal={arXiv preprint arXiv:2408.15417},
  year={2024}
}

@article{mixon2022neural,
  title={Neural collapse with unconstrained features},
  author={Mixon, Dustin G and Parshall, Hans and Pi, Jianzong},
  journal={Sampling Theory, Signal Processing, and Data Analysis},
  volume={20},
  number={2},
  pages={11},
  year={2022},
  publisher={Springer}
}

@article{saxe2019mathematical,
  title={A mathematical theory of semantic development in deep neural networks},
  author={Saxe, Andrew M and McClelland, James L and Ganguli, Surya},
  journal={Proceedings of the National Academy of Sciences},
  volume={116},
  number={23},
  pages={11537--11546},
  year={2019},
  publisher={National Acad Sciences}
}

@article{liu2024exploration,
  title={The Exploration of Neural Collapse under Imbalanced Data},
  author={Liu, Haixia},
  journal={arXiv preprint arXiv:2411.17278},
  year={2024}
}

@article{bach_saxe,
  title={Implicit regularization of discrete gradient dynamics in linear neural networks},
  author={Gidel, Gauthier and Bach, Francis and Lacoste-Julien, Simon},
  journal={Advances in Neural Information Processing Systems},
  volume={32},
  year={2019}
}

@article{saxe2013exact,
  title={Exact solutions to the nonlinear dynamics of learning in deep linear neural networks},
  author={Saxe, Andrew M and McClelland, James L and Ganguli, Surya},
  journal={arXiv preprint arXiv:1312.6120},
  year={2013}
}

@article{hong,
  title={Neural collapse for unconstrained feature model under cross-entropy loss with imbalanced data},
  author={Hong, Wanli and Ling, Shuyang},
  journal={arXiv preprint arXiv:2309.09725},
  year={2023}
}

@inproceedings{behnia2023implicit,
  title={On the Implicit Geometry of Cross-Entropy Parameterizations for Label-Imbalanced Data},
  author={Behnia, Tina and Kini, Ganesh Ramachandra and Vakilian, Vala and Thrampoulidis, Christos},
  booktitle={International Conference on Artificial Intelligence and Statistics},
  pages={10815--10838},
  year={2023},
  organization={PMLR}
}

@article{seli,
  title={Imbalance Trouble: Revisiting Neural-Collapse Geometry},
  author={Thrampoulidis, Christos and Kini, Ganesh R and Vakilian, Vala and Behnia, Tina},
  journal={arXiv preprint arXiv:2208.05512},
  year={2022}
}

@article{lyu2019gradient,
  title={Gradient descent maximizes the margin of homogeneous neural networks},
  author={Lyu, Kaifeng and Li, Jian},
  journal={arXiv preprint arXiv:1906.05890},
  year={2019}
}

@article{zhu2021geometric,
  title={A Geometric Analysis of Neural Collapse with Unconstrained Features},
  author={Zhu, Zhihui and Ding, Tianyu and Zhou, Jinxin and Li, Xiao and You, Chong and Sulam, Jeremias and Qu, Qing},
  journal={Advances in Neural Information Processing Systems},
  volume={34},
  year={2021}
}

@article{han2021neural,
  title={Neural collapse under mse loss: Proximity to and dynamics on the central path},
  author={Han, XY and Papyan, Vardan and Donoho, David L},
  journal={arXiv preprint arXiv:2106.02073},
  year={2021}
}

@article{tirer2022extended,
  title={Extended unconstrained features model for exploring deep neural collapse},
  author={Tirer, Tom and Bruna, Joan},
  journal={arXiv preprint arXiv:2202.08087},
  year={2022}
}

@article{mixon2020neural,
  title={Neural collapse with unconstrained features},
  author={Mixon, Dustin G and Parshall, Hans and Pi, Jianzong},
  journal={arXiv preprint arXiv:2011.11619},
  year={2020}
}

@article{fang2021exploring,
  title={Exploring deep neural networks via layer-peeled model: Minority collapse in imbalanced training},
  author={Fang, Cong and He, Hangfeng and Long, Qi and Su, Weijie J},
  journal={Proceedings of the National Academy of Sciences},
  volume={118},
  number={43},
  year={2021},
  publisher={National Acad Sciences}
}

@inproceedings{gunasekar2018characterizing,
  title={Characterizing implicit bias in terms of optimization geometry},
  author={Gunasekar, Suriya and Lee, Jason and Soudry, Daniel and Srebro, Nathan},
  booktitle={International Conference on Machine Learning},
  pages={1832--1841},
  year={2018},
  organization={PMLR}
}

@misc{CDT,
      title={Identifying and Compensating for Feature Deviation in Imbalanced Deep Learning}, 
      author={Han-Jia Ye and Hong-You Chen and De-Chuan Zhan and Wei-Lun Chao},
      year={2020},
      eprint={2001.01385},
      archivePrefix={arXiv},
      primaryClass={cs.LG}
}

@ARTICLE{KimKim,  author={Kim, Byungju and Kim, Junmo},  journal={IEEE Access},   title={Adjusting Decision Boundary for Class Imbalanced Learning},   year={2020},  volume={8},  number={},  pages={81674-81685},  doi={10.1109/ACCESS.2020.2991231}}

@inproceedings{byrd2019effect,
  title={What is the effect of importance weighting in deep learning?},
  author={Byrd, Jonathon and Lipton, Zachary},
  booktitle={International Conference on Machine Learning},
  pages={872--881},
  year={2019},
  organization={PMLR}
}

@article{gunasekar2018implicit,
  title={Implicit Bias of Gradient Descent on Linear Convolutional Networks},
  author={Gunasekar, Suriya and Lee, Jason D and Soudry, Daniel and Srebro, Nati},
  journal={Advances in Neural Information Processing Systems},
  volume={31},
  pages={9461--9471},
  year={2018}
}

@article{soudry2018implicit,
  title={The implicit bias of gradient descent on separable data},
  author={Soudry, Daniel and Hoffer, Elad and Nacson, Mor Shpigel and Gunasekar, Suriya and Srebro, Nathan},
  journal={The Journal of Machine Learning Research},
  volume={19},
  number={1},
  pages={2822--2878},
  year={2018},
  publisher={JMLR. org}
}

@inproceedings{ji2019implicit,
  title={The implicit bias of gradient descent on nonseparable data},
  author={Ji, Ziwei and Telgarsky, Matus},
  booktitle={Conference on Learning Theory},
  pages={1772--1798},
  year={2019},
  organization={PMLR}
}

@inproceedings{sagawa2020investigation,
  title={An investigation of why overparameterization exacerbates spurious correlations},
  author={Sagawa, Shiori and Raghunathan, Aditi and Koh, Pang Wei and Liang, Percy},
  booktitle={International Conference on Machine Learning},
  pages={8346--8356},
  year={2020},
  organization={PMLR}
}

@article{li2021autobalance,
  title={AutoBalance: Optimized Loss Functions for Imbalanced Data},
  author={Li, Mingchen and Zhang, Xuechen and Thrampoulidis, Christos and Chen, Jiasi and Oymak, Samet},
  journal={Advances in Neural Information Processing Systems},
  volume={34},
  pages={3163--3177},
  year={2021}
}

@inproceedings{TengyuMa,
  title={Learning imbalanced datasets with label-distribution-aware margin loss},
  author={Cao, Kaidi and Wei, Colin and Gaidon, Adrien and Arechiga, Nikos and Ma, Tengyu},
  booktitle={Advances in Neural Information Processing Systems},
  pages={1567--1578},
  year={2019}
}

@article{Menon,
  title={Long-tail learning via logit adjustment},
  author={Menon, Aditya Krishna and Jayasumana, Sadeep and Rawat, Ankit Singh and Jain, Himanshu and Veit, Andreas and Kumar, Sanjiv},
  journal={arXiv preprint arXiv:2007.07314},
  year={2020}
}

\onecolumn
\section{Background on Loss Reweighting}
\label{sec:reweight_background}

\subsection{Statistical Optimality in the Classical Regime}\label{sec:population}

Loss reweighting is theoretically grounded, being statistically optimal in the population 
limit for optimizing balanced accuracy, which weighs class-conditional 
accuracies equally rather than by class frequencies.
Concretely, let $\pi_c$ be the prior probability of class 
$c\in[k]$ in a $k$-class classification dataset. The standard 
error of a hypothesis $h$ is 
$\Rc(h):=\sum_{c\in[k]}\pi_c \Rc_c(h)=
\mathbb{E}_{c\sim P(c), \x\sim P(\x|c)}[\mathbf{1}[h(\x)\neq c]],$
where $\Rc_c(h)=\mathbb{E}_{\x\sim P(\x|c)}[\mathbf{1}[h(\x)
\neq c]]$ is the class-conditional error. 
In practice, given $n$ samples $(\x_i)_{i=1}^n$ 
with labels $y_i \in [k]$, we minimize the empirical risk 
\begin{align}\label{eq:loss standard}
\Lc(h) = \frac{1}{n}\sum_{i=1}^n\ell\left(h(\x_i), y_i\right),
\end{align}
where $\ell$ denotes a differentiable proxy to the zero-one loss.
The balanced error is $\Rc_{\text{bal}}(h):=\frac{1}{k}\sum_{c\in[k]}
\Rc_c(h)$, with empirical counterpart
$\Lc_{\text{bal}}(h) = \frac{1}{k}\sum_{c\in[k]}
\frac{1}{n_c}\sum_{i:y_i=c}\ell\left(h(\x_i), c\right)$,
where $n_c=|\{i:y_i=c\}|$. This can be rewritten as a reweighted loss:
\begin{align}\label{eq:loss weighted}
\Lc_{\text{reweight}}(h) = \frac{1}{n}\sum_{i=1}^n \omega_i \cdot \ell(h(\x_i), y_i),
\end{align}
where $\omega_i=1/(n_{y_i}/n)$ is the inverse empirical frequency of each class.
Under mild assumptions, this re-weighted loss is statistically optimal when 
$n\rightarrow\infty$ and the hypothesis space is bounded.

\subsection{Nuanced Picture in the Overparameterization Regime}\label{sec:literature}

However, modern machine learning challenges the (classical) population analysis: 
neural networks operate in high-dimensional regimes where the hypothesis-space 
dimension exceeds the training sample count. In such overparameterized settings, 
classical methods often lose their optimality. Indeed, several works 
\citep{byrd2019effect,KimKim} have identified weighted cross-entropy's 
failure to substantially improve balanced accuracy of well-trained overparameterized 
deep neural networks on imbalanced datasets.

Investigating this behavior theoretically, \citet{sagawa2020investigation} and \citet{kini2021label} attribute this failure to the implicit optimization bias of gradient-based optimizers used to minimize the loss during training. A long line of research \citep{gunasekar2018implicit,gunasekar2018characterizing,soudry2018implicit,ji2019implicit,lyu2019gradient} has supported the idea that empirical risk minimization of overparameterized networks—where many minimizers of the loss can exist—is empirically successful precisely because first-order gradient optimizers are biased toward solutions that maximize the margin over the training data among all train loss minimizers. Applied to weighted cross-entropy, \citet{sagawa2020investigation} and \citet{kini2021label} leverage this type of analysis to demonstrate that loss reweighting has no effect on the implicit optimization bias: irrespective of the choice of weights, when trained long enough, the configuration chosen by gradient-based optimization applied to weighted cross-entropy is the same as vanilla cross-entropy.

\citet{kini2021label} identified that in this overparameterized regime reweighting still works, but should be done differently than suggested by Eq. \eqref{eq:loss weighted} and should instead be applied to the logits. See also related \citep{TengyuMa,Menon,xu2021understanding,welfert2024theoretical} and follow-up works e.g., \cite{behnia2022avoid,behnia2023implicit,wang2021importance,lai2024sharp,mor2025analytical,wang2023unified,zhai2022understanding}.

Despite being very insightful and leading to practical modifications, there are two limitations in the above implicit bias analysis: First, the implicit optimization bias analysis provides explicit characterizations of the max-margin classifiers to which optimizers converge only in limited settings (even for two-layer linear networks, it has not been rigorously shown whether gradient descent converges to the global minimizer of the non-convex max-margin optimization). Second, and perhaps more importantly, the implicit bias toward a max-margin classifier manifests very late in training: even in linear models, convergence to the max-margin solution with gradient descent is exponentially slow (convergence can be accelerated with normalization, but still requires many iterations).

On the other hand, it has been empirically observed that loss reweighting can actually be beneficial when either used with early stopping or as a complementary technique to modern loss variants optimized for the overparameterized regime, leading to empirical speed-ups and optimization boosts in early training iterations \citep{TengyuMa,xu2021understanding,li2021autobalance}. 
\section{Additional Experimental Details on Motivating Example}
\label{sec:additional_exp}
\begin{figure*}
    \centering
    \includegraphics[width=1\linewidth]{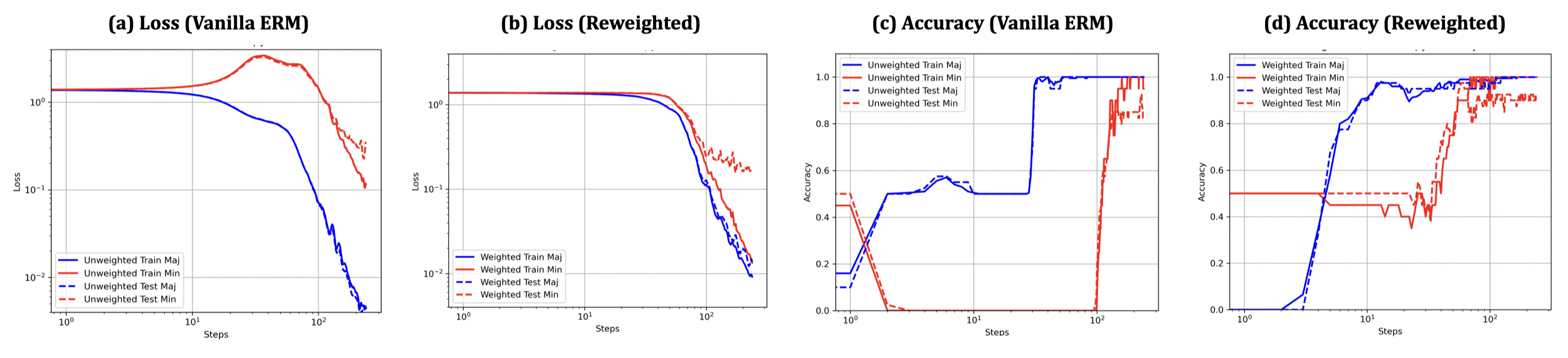}
    \caption{Training dynamics of majority vs. minority classes on an imbalanced 4-class MNIST task (\( R = 10 \), 2 majority, 2 minority). 
    \textbf{(a)} Under \textit{vanilla ERM}, minority loss increases at early stage and remains higher in training, while majority loss quickly drops. 
    \textbf{(b)} Under \textit{reweighted ERM}, both majority and minority losses decrease together, indicating balanced optimization. 
    \textbf{(c)} Vanilla ERM results in delayed learning for minorities, as majority accuracy increases early and minority accuracy only improves after 100 steps. 
    \textbf{(d)} Reweighting leads to a relatively synchronized accuracy gain for both groups, with minority test accuracy improving much earlier. }

    \label{fig:mnist_loss_acc}
\end{figure*}
\subsection{Loss Trajectories}
Figure~\ref{fig:mnist_loss_acc} (a)-(b) shows the loss curves for minority and majority classes.

\subsection{Accuracy Dynamics}
Figure~\ref{fig:mnist_loss_acc} (c)-(d) plots accuracies.

\subsection{Confusion Matrix Progression on Test Data}

To further illustrate the difference in generation ability  between standard and reweighted training, we visualize the evolution of test-set confusion matrices over training steps (Figure~\ref{fig:confmat_test}). Each matrix shows class-level prediction accuracy at a given step.

\begin{figure*}[h]
    \centering
    \includegraphics[width=\textwidth]{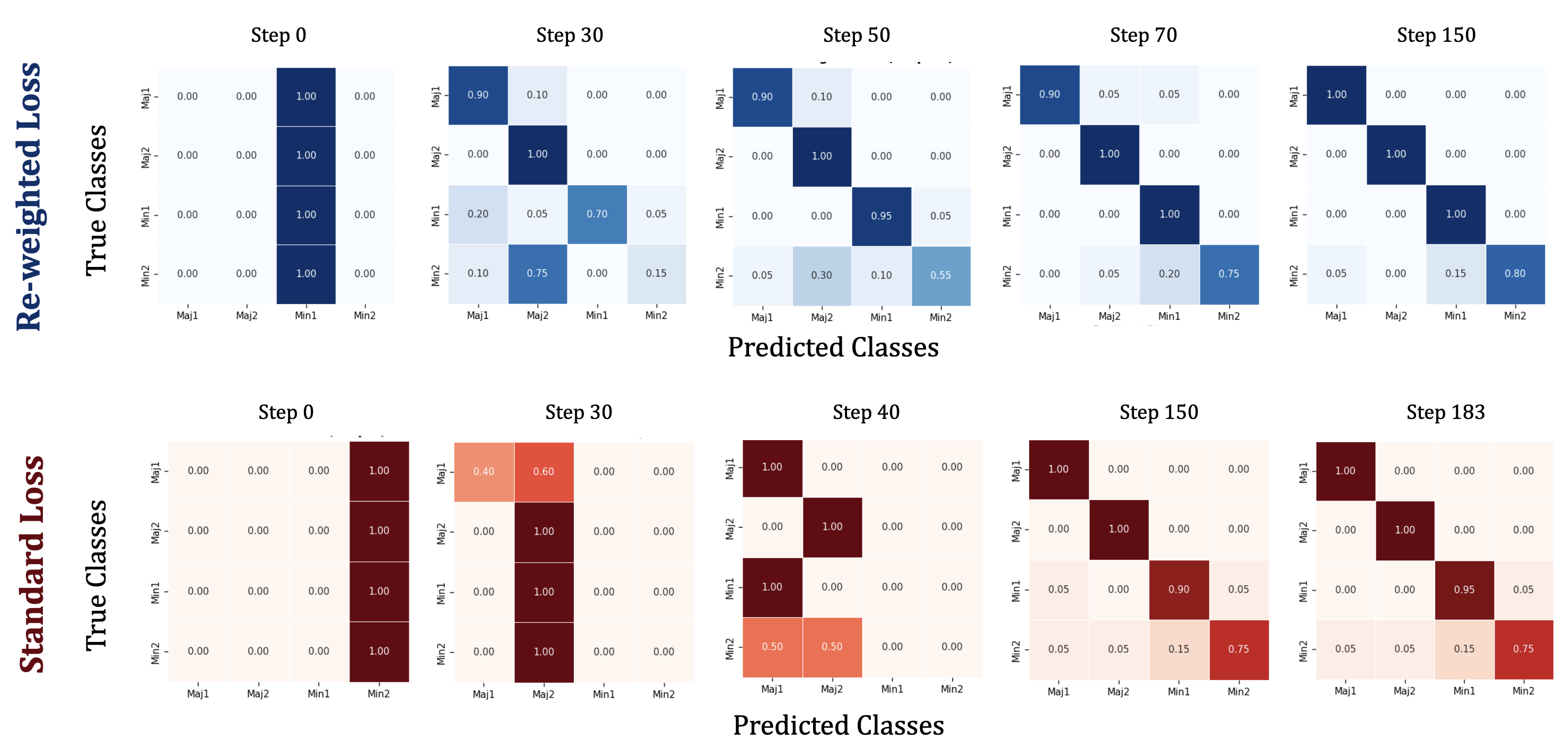}
    \caption{Confusion matrix progression on test data under reweighted (top row) vs.\ standard (bottom row) cross-entropy loss. Each matrix shows predictions across 4 classes (2 majority, 2 minority) at selected training steps. Under standard loss, the model first learns majority classes, with minority learning delayed. Reweighting enables earlier and more balanced classification across all classes.}
    \label{fig:confmat_test}
\end{figure*}

These trends validate our spectral hypothesis: \emph{vanilla CE follows the ordering of singular values of the label matrix—favoring majority-majority features—while Re-weighted flattens this spectrum and equalizes learning rates across all concept directions.}

\subsection{Sensitivity to Digit Selection}
We observe that the specific digits chosen as majority vs.\ minority classes can influence learning dynamics. For instance, using \texttt{maj} = (2,3), \texttt{min} = (0,1) leads to unusually fast minority learning under reweighted loss—minorities are correctly classified from the start. In contrast, reversing the roles (\texttt{maj} = (0,1), \texttt{min} = (2,3)) delays minority learning.

This asymmetry likely reflects differences in digit difficulty (e.g., 0 and 1 are easier to distinguish than 2 and 3). Importantly, across all configurations tested, loss reweighting consistently accelerates minority learning relative to standard loss, even if the absolute ordering varies.

\section{Analysis Details}\label{app:saxe_connections}

\subsection{Background on Vanilla ERM}
Consider the square-loss UFM:
\begin{align}\label{eq:UFM_L2}
    \min_{\W,\Hb}~~\left\{\frac{1}{2}\sum_{i\in[n]}\|\sbar_i-\W\hb_i\|^2=\frac{1}{2}\|\Sbar-\W\Hb\|^2\right\}\,.
\end{align}
which fits logits $\W\Hb$ to the centered sparsity matrix $\Sbar$. Although CE loss is more common in practice, we note that square-loss has shown competitive performance to CE minimization in various settings \cite{hui2020evaluation,demirkaya2020exploring}.
The neural-collapse literature has extensively studied Eq. \eqref{eq:UFM_L2} primarily in the balanced case (e.g., \cite{mixon2020neural,han2021neural,sukenik2023deep,tirer2022extended}) but recently also for imbalanced data (e.g., \cite{liu2024exploration,hong}). Most works focus on global minima of regularized UFM, with less attention to unregularized cases or training dynamics. While some landscape analyses provide partial answers about global convergence \citet{mixon2022neural,han2021neural}, they are limited to regularized cases and do \emph{not} characterize dynamics. For example, \cite{han2021neural}'s analysis of the 'central path' in balanced one-hot cases—these results relies on approximations. Thus, a significant gap remains in understanding UFM training dynamics, even for simple balanced one-hot data with square loss.

By interpreting the UFM with square loss in Eq. \eqref{eq:UFM_L2} as a two-layer linear network with orthogonal inputs, we identify a connection to \citet{saxe2013exact,bach_saxe}'s analysis, that to the best of our knowledge has thus far remained unexplored in the neural-collapse literature. \citet{saxe2013exact} provide explicit characterization of gradient descent dynamics (with small initialization) for square-loss UFM. The key insight in adopting their results, is rewriting \eqref{eq:UFM_L2} as $\sum_{i\in[n]}\|\sbar_i-\W\Hb\eb_i\|^2$ with \emph{orthogonal} inputs $\eb_i\in\R^m$. This enables direct application of their result, originally stated in \cite{saxe2013exact} and formalized in \cite{bach_saxe}. For completeness, we state this here in our setting and terminology as a proposition below.
\begin{proposition}[\citep{saxe2013exact,bach_saxe}]
    \label{thm:saxe} Consider gradient flow (GF) dynamics for minimizing the square-loss UFM \eqref{eq:UFM_L2}. Recall the SVD $\Sbar=\Ub\Sigmab\Vb^\top$. Assume weight initialization 
\[
\W(0)=e^{-\delta}\Ub\Rb^\top \quad\text{and}\quad \Hb(0)=e^{-\delta}\Rb\Vb^\top
\]
for some partial orthogonal matrix $\Rb\in\R^{d\times r}$ ($\Rb^\top\Rb=\Id_r$) and initialization scale $e^{-\delta}$. Then the iterates   $\W(t),\Hb(t)$ of GF are as follows:
\begin{align}
    \W(t) = \Ub\sqrt{\Sigmab}\sqrt{\Ab(t)}\Rb^\top \qquad\text{and} \qquad \Hb(t) = \Rb\sqrt{\Sigmab}\sqrt{\Ab(t)}\Vb^\top
\end{align}
for $\Ab(t)=\diag{a_1(t),\ldots,a_r(t)}$ with
\begin{align}\label{eq:a_i(t) exact}
    a_i(t)=\frac{1}{ 1 + ( \sigma_i e^{2\delta} -1) e^{-2\sigma_i t}}
    , ~i\in[r].
\end{align}
Moreover, the time-rescaled factors $a_i(\delta t)$ converge to a step function as $\delta\rightarrow\infty$ (limit of vanishing initialization):
\begin{align}
    a_i(\delta t)\rightarrow \frac{1}{1+\sigma_i}\ones[{t=T_i}]+ \ones[{t>T_i}],
\end{align}
where $T_i=1/\sigma_i$ and $\ones[{A}]$ is the indicator function for event $A$. Thus, the $i$-th component is learned at time $T_i$ inversely proportional to $\sigma_i$.
\end{proposition}
\begin{proof}
    After having set up the analogy of our setting to that of \cite{saxe2013exact,bach_saxe}, this is a direct application of \citep[Thm.~1]{bach_saxe}. Specifically, this is made possible in our setting by: (i) interpreting the UFM with square-loss in \eqref{eq:UFM_L2} as a two-layer linear network (ii) recognizing that the covariance of the inputs (which here are standard basis vectors $\eb_j, j\in\R^m$) is the identity matrix, hence the (almost) orthogonality assumption (see \citep{saxe2013exact} and \citep[Sec.~4.1]{bach_saxe}) holds.
\end{proof}

This result requires initializing $\W,\Hb$ in a way that makes them aligned with the SVD factors of the SEL matrix. While this might appear as a strong assumption, \cite{saxe2013exact,saxe2019mathematical} conjectured and verified experimentally that the characterization remains qualitatively accurate under small random initialization. 
Our experiments with the SEL matrix confirm this - Fig. \ref{fig:convergence with weights} (middle row) shows the singular values of the logit matrix during training closely follow the predicted exponential trend in Eq. \eqref{eq:a_i(t) exact}. This reveals that dominant singular factors, corresponding to primary semantic concepts, are learned first. In the limit $t\rightarrow\infty$, the theorem shows convergence to:
\begin{align}\label{eq:Winf}
\W(t) \rightarrow \Winf:=\Ub\sqrt{\Sigmab}\Rb^\top \qquad\text{and} \qquad \Hb(t) \rightarrow \Hbinf:=\Rb\sqrt{\Sigmab}\Vb^\top,.
\end{align}
This aligns with \cite{seli}'s  regularization-path analysis of UFM with CE loss, where normalized quantities converge as regularization goes to zero.

\begin{figure*}[t]
   \centering
\hspace{-0.38in}   \includegraphics[width=1\textwidth]{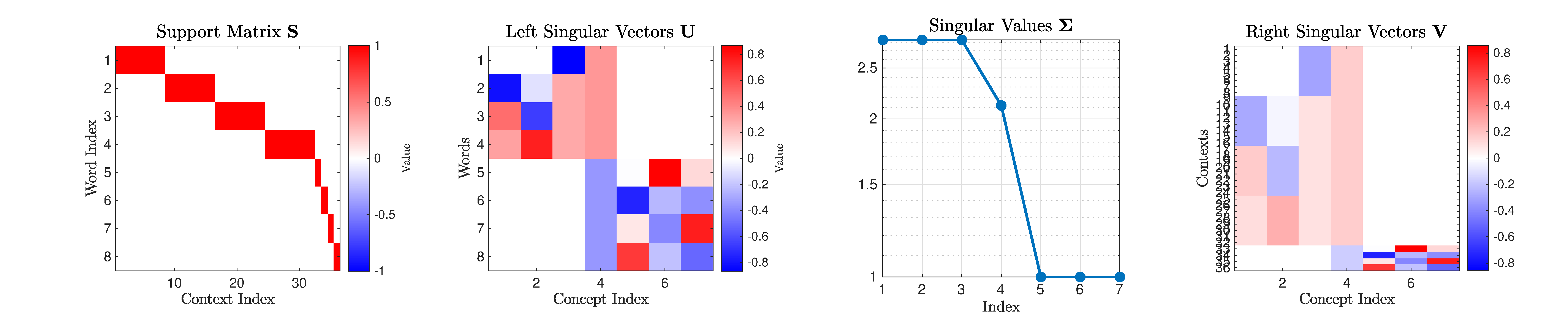}

   \vspace{1em}
\hspace{-0.38in}\includegraphics[width=1\textwidth]{figs_christos/imbalanced_without_weights.png}

\vspace{1em}\hspace{-0.38in}\includegraphics[width=1\textwidth]{figs_christos/imbalanced_with_weights.png}
  \captionsetup{width=\textwidth} \caption{ \textbf{(Top)} One-hot encoding matrix and SVD factors of SEL matrix for  STEP-imbalanced data. \textbf{(Middle)} Training dynamics of GD minimization of UFM with square loss (Eq. \eqref{eq:UFM_L2}. for two initializations: (i) SVD: initialize $\W$ and $\Hb$ as per Thm. \ref{thm:saxe} for $\delta=8$. (ii) Rand: intialize $\W$ and $\Hb$ random Gaussian scaled to match the norm of SVD initialization. Dynamics with the two initialization are shown in orange (SVD) and blue (Rand), respectively. Qualitatively the behavior is similar. With vanilla ERM, SVD factors are learned in the order of their singular values, see Fig.~~\ref{fig:convergence} for a more general synthetic setting. \textbf{(Bottom)} Training dynamics of GD minimization of UFM with weighted square loss (Eq. \eqref{eq:weighted UFM_L2}) with weights as in Eq. \eqref{eq:weights STEP}. Note that thanks to the weighting all singular factors are now learned at approximately the same rates.}
   \label{fig:convergence with weights}

\end{figure*}

\subsection{Controlling the rate of learning via Reweighting} \label{app:reweight}
Consider minimizing the following weighted version of \eqref{eq:UFM_L2}:
\begin{align}\label{eq:weighted UFM_L2}    \min_{\W,\Hb}~~\left\{\frac{1}{2}\sum_{i\in[n]} \omega_i \|\z_i-\W\hb_i\|_2^2 = \frac{1}{2}\left\|\left(\Sbar-\W\Hb\right)\Omegab^{1/2}\right\|_F^2\right\}\,,
\end{align}
where $\Omegab=\diag{[\omega_1,\ldots,\omega_m]}$ is a diagonal matrix of weights, one for each context. We consider the STEP-imbalanced one-hot classification setting described. Concretely, let the one-hot label matrix 
be
\begin{align}\label{eq:STEP}
\smat = \begin{bmatrix}
    \Id_{k/2}\otimes \ones_R^\top & \mathbf{0}_{k/2 \times k/2}
    \\
    \mathbf{0}_{k/2 \times Rk/2} & \Id_{k/2}
\end{bmatrix}
\end{align}
where $R$ is the imbalance ratio and without loss of generality we assumed that the first $k/2$ classes are majorities and that minorities have $1$ example each. Here, the choice of $1$ sample per minority is done without loss of generality and just to maintain simplicity in the formulas. Thus, the total number of examples is $n=Rk/2+k/2=(R+1)k/2$. Recall that $\smatbar=(\Id_k-\frac{1}{k}\ones_k\ones_k^\top)\smat$.

Recall that $\smatbar=\Ub\Sigmab\Vb^\top$. \citet[Lem.~A.3]{seli} derives the spectral components in closed form as follows:
\begin{align}\label{eq:Sigmab}
\Sigmab=\diag{\begin{bmatrix}\sqrt{R}\ones_{k/2-1} & \sqrt{(R+1)/2} & \ones_{k/2-1}\end{bmatrix}}\,
\end{align}
\begin{align}\nn
\Ub = \left[\begin{array}{ccc}
\mathbb{F} & -\sqrt{\frac{1}{k}}\ones & \mathbf{0} \\
\mathbf{0} & \sqrt{\frac{1}{k}}\ones & \mathbb{F}
\end{array}\right]\in\R^{k\times (k-1)}\,,
\end{align}
and
\[
\Vb^\top = \left[\begin{array}{ccc}
\sqrt\frac{1}{R}\mathbb{F}^\top\otimes\ones_R^\top & \mathbf{0} 
\\-\sqrt{\frac{2}{(R+1)k}}\ones_{Rk/2}^\top  & \sqrt{\frac{2}{(R+1)k}}\ones_{k/2}^\top \\ 
\mathbf{0} & \mathbb{F}^\top
\end{array}\right]\in\R^{(k-1)\times m}\,.
\]
Above,  $\mathbb{F}\in\R^{k/2\times(k/2-1)}$ is an orthonormal basis of the subspace orthogonal to $\ones_{k/2}$.

We now set the weights inversely proportional to the square-root of the respective class frequency, i.e.  the weight for the first $Rk/2$ majority samples is $\sqrt{n/R} = \sqrt{(R+1)k/(2R)}$, while for the minorities is $\sqrt{R}$ times larger, i.e.,  $\sqrt{n}=\sqrt{(R+1)k/2}$. The relative rates of learning of majority/minority features do not change if we uniformly scale all weights by a constant. Thus, for simplicity and without affecting our analysis, we drop the factor $\sqrt{k/2}$, and, in matrix form, choose
\begin{align}\label{eq:weights STEP}
\Omegab:=\sqrt{R+1}\cdot\diag{\begin{bmatrix} \sqrt\frac{1}{R}\,\ones_{R k/2}^\top & \ones_{k/2}^\top \end{bmatrix}}
\,.
\end{align}

Direct calculation yields the following:
\begin{align}
\Vb^\top\Omegab\Vb &= \sqrt{R+1}\left[\begin{array}{ccc}
{\frac{1}{R}}\mathbb{F}^\top\otimes\ones_R^\top & \mathbf{0} 
\\-\sqrt{\frac{2}{R(R+1)k}}\ones_{Rk/2}^\top  & \sqrt{\frac{2}{(R+1)k}}\ones_{k/2}^\top \\ 
\mathbf{0} & \mathbb{F}^\top
\end{array}\right] \begin{bmatrix}
\sqrt\frac{1}{R}\mathbb{F}\otimes\ones_R & -\sqrt{\frac{2}{(R+1)k}}\ones_{Rk/2} & \zeros
\\
\zeros & \sqrt{\frac{2}{(R+1)k}}\ones_{k/2} & \mathbb{F}
\end{bmatrix}
\nn\\
&= 
\begin{bmatrix}
    \sqrt\frac{R+1}{R}\,\Id_{k/2-1} &\zeros & \zeros
    \\
    \zeros & \frac{\sqrt{R}+1}{\sqrt{R+1}} &\zeros
    \\
    \zeros & \zeros & \sqrt{R+1} \,\Id_{k/2-1}
\end{bmatrix} =:\Lambdab\label{eq:lambdab def}
.
\end{align}
where in the last equation, we defined the diagonal matrix $\Lambdab$ with entries:
\begin{align}\label{eq:lambdab def2}
    \lambda_i:=\begin{cases}
        \sqrt{\frac{R+1}{R}} & i\in[k/2-1]
        \\
        \frac{\sqrt{R}+1}{\sqrt{R+1}} & i=k/2
        \\
        \sqrt{R+1} & i=k/2+1,\ldots,k-1
    \end{cases}
\end{align}
This calculation is handy in the proof of the following theorem which is our main theoretical result. The proof shows that $\Lambdab$ is the effective weight matrix applied to the spectrum of the label matrix.



\begin{theorem}\label{thm:ours} Consider gradient flow (GF) dynamics for minimizing the \textbf{weighted} square-loss UFM \eqref{eq:weighted UFM_L2} with weight matrix as shown in Eq. \eqref{eq:weights STEP} in an $R$-STEP-imbalanced setting. Assume same parameter initialization as in Proposition \ref{thm:saxe}, i.e.,
$
\W(0)=e^{-\delta}\Ub\Rb^\top\,,\,\Hb(0)=e^{-\delta}\Rb\Vb^\top$, for partial orthogonal matrix $\Rb\in\R^{d\times (k-1)}$, initialization scale $e^{-\delta}$ and SVD $\Sbar=\Ub\Sigmab\Vb^\top$. Finally, consider the \textbf{effective weights} $\lambda_i, i\in[k-1]$ defined in Eq. \eqref{eq:lambdab def2}. Then the iterates   $\W(t),\Hb(t)$ of GF evolve as follows:
\begin{align}
    \W(t) = \Ub\sqrt{\Sigmab}\sqrt{\Bb(t)}\Rb^\top \qquad\text{and} \qquad \Hb(t) = \Rb\sqrt{\Sigmab}\sqrt{\Bb(t)}\Vb^\top
\end{align}
for $\Bb(t)=\diag{\beta_1(t),\ldots,\beta_{k-1}(t)}$ with
\begin{align}\label{eq:a_i(t) exact}
    \beta_i(t)=\frac{1}{ 1 + ( \sigma_i e^{2\delta} -1) e^{-2\sigma_i \lambda_i t}}
    , ~i\in[k-1].
\end{align}
Moreover, the time-rescaled factors $\beta_i(\delta t)$ converge to a step function as $\delta\rightarrow\infty$:
\begin{align}
    \beta_i(\delta t)\rightarrow \frac{1}{1+\sigma_i}\ones[{t=T_i}]+ \ones[{t>T_i}],
\end{align}
where $T_i=1/(\sigma_i\lambda_i)$. Thus, the $i$-th component is learned at time  inversely proportional to $\lambda_i\cdot\sigma_i$.
\end{theorem}
\begin{proof}
 The GF updates are given by:
\begin{subequations}\label{eq:original weight}
\begin{align}
    \dt{\W(t)} &= -\left(\smatbar\Omegab^{1/2}-\W(t)\Hb(t)\Omegab^{1/2}\right)\Omegab^{1/2}\Hb(t)^\top =
    -\left(\smatbar-\W(t)\Hb(t)\right)\Omegab\Hb(t)^\top
    \\
        \dt{\Hb(t)} &= -\W(t)^\top\left(\smatbar\Omegab^{1/2}-\W(t)\Hb(t)\Omegab^{1/2}\right) \Omegab^{1/2} = -\W(t)^\top\left(\smatbar-\W(t)\Hb(t)\right) \Omegab
\end{align}
\end{subequations}
where $\W(0)=e^{-\delta}\Ub\Rb^\top$ and $\Hb(0)=e^{-\delta}\Rb\Vb^\top$. As in \cite{saxe2013exact,saxe2019mathematical} change variables to $\Wbar(t), \Hbar(t)$ defined as
\[
\W(t)=\Ub\Wbar(t)\Rb^\top, \qquad\text{and}\qquad\Hb(t)=\Rb\Hbar(t)\Vb^\top\,.
\]
At $t=0$, these matrices are diagonal and equal to $e^{-\delta}\Id$ by initialization assumption. Substituting these in the update equations \eqref{eq:original weight} and also using $\smatbar=\Ub\Sigmab\Vb^\top$, we get that
\begin{subequations}\label{eq:original weight 2}
\begin{align}
    \Ub\dt{\Wbar(t)}\Rb^\top &= 
    \Ub\left(\Sigmab-\Wbar(t)\Hbar(t)\right)\Vb^\top\Omegab\Vb\Hbar(t)^\top\Rb^\top
    \\
        \Rb\dt{\Hbar(t)}\Vb^\top &= \Rb\Wbar(t)^\top\left(\Sigmab-\Wbar(t)\Hbar(t)\right)\Vb^\top \Omegab\,.
\end{align}
\end{subequations}
Left and Right multiplying these with the partial unitary matrices $\Ub,\Rb,\Vb$ and recalling from \eqref{eq:lambdab def} that $\Vb^\top\Omegab\Vb=\Lambdab$ (given in \eqref{eq:lambdab def2}), we arrive at
\begin{subequations}\label{eq:original weight 3}
\begin{align}
    \dt{\Wbar(t)} &= -
    \left(\Sigmab-\Wbar(t)\Hbar(t)\right)\Lambdab\Hbar(t)^\top
    \\
        \dt{\Hbar(t)}&= - \Wbar(t)^\top\left(\Sigmab-\Wbar(t)\Hbar(t)\right)\Lambdab\,.
\end{align}
\end{subequations}
At this point, recall that $\Lambdab$ is diagonal, and of course so is $\Sigmab$. Since $\Wbar,\Hbar$ are initialized to be diagonal matrices, they will remain diagonal throughout training. We thus arrive at a decoupled system of differential equations as in the unweighted case treated by \cite{saxe2013exact,bach_saxe}. 

In fact, the update equations are same as in the unweighted case (compare to \citep[Eq. 39]{bach_saxe}) 
with the only difference being the presence of the \emph{effective weight matrix} $\Lambdab$. Having reached this point, we can follow almost the exact same steps as in the proof of \citep[Thm. 1]{bach_saxe}. For brevity and because the adjustments are easy to check we directly present the formula for the diagonal entries $\wbars_i(t),\hbars_i(t)$ of $\Wbar(t), \Hbar(t)$ that solve \eqref{eq:original weight 3} (see \citep[Eq. 48]{bach_saxe}). For $i\in[k-1]$:
\begin{align}
    \wbars_i(t) = \hbars_i(t) = \sqrt{\frac{\sigma_i e^{2\sigma_i\lambda_it}}{\sigma_i e^{2\delta} - 1 + e^{2\sigma_i\lambda_i t}}} = \sqrt{\sigma_i}\sqrt\frac{1}{1+(\sigma_i e^{2\delta} - 1)e^{-2\sigma_i\lambda_i t}} =: \sqrt{\sigma_i}\sqrt{\beta_i(t)}\,,
\end{align}
where we defined
\[
\beta_i(t) := \frac{1}{1+(\sigma_i e^{2\delta} - 1)e^{-2\sigma_i\lambda_i t}} \,,i\in[k-1]
\]
Note the similarity to the unweighted case (Eq. \eqref{eq:a_i(t) exact}) only here the rate at which $\beta_i(t)$ approaches $1$ is determined by $\sigma_i\lambda_i$, rather than by $\sigma_i$ alone. 

From this, we can also easily determine the limit of  $\beta_i(\delta t)$ as $\delta\rightarrow\infty$ (see \citep[Eq. 64]{bach_saxe}):
\begin{align}
    \beta_i(\delta t) := \frac{e^{2\sigma_i\lambda_i \delta t}}{e^{2\sigma_i\lambda_i \delta t}+(\sigma_i e^{2\delta} - 1)} = \frac{e^{2\delta (\sigma_i\lambda_i t - 1)}}{e^{2 \delta (\sigma_i\lambda_i t-1)}+(\sigma_i  - e^{-2\delta})} \stackrel{\delta\rightarrow\infty}{\longrightarrow} \frac{1}{1+\sigma_i}\ones[{t=\frac{1}{\sigma_i\lambda_i}}]+ \ones[{t>\frac{1}{\sigma_i\lambda_i}}]\,.
\end{align}
\end{proof}

Combining Proposition \ref{thm:saxe} (for vanilla ERM) with Theorem \ref{thm:ours} (for reweighting) we arrive at the following corollary comparing the relative rates of learning of \majmaj, \majmin, \minmin features.
\begin{corollary}\label{cor:delta}
    Assume the spectral initialization of Prop. \ref{thm:saxe} and Thm. \ref{thm:ours} with vanishing initialization. Both vanillaERM and reweighting dynamics show a phase transition. On the one hand, Vanilla ERM learns the features in the order \majmaj, \majmin, \minmin at times $1/\sqrt{R}$, $1/\sqrt{(R+1)/2}$ and $1$, respectively. On the other hand, weigthed ERM \eqref{eq:weighted UFM_L2} learns simultaneously the features \majmaj and \minmin at time $1/\sqrt{R+1}$ before they learn the $\majmin$ feature, closely after, at time $\sqrt{2}/(\sqrt{R}+1)$.
\end{corollary}
\begin{proof}
    This is direct consequence of the phase-transition like rates of Prop. \ref{thm:saxe} and Thm \ref{thm:ours} at infinitesimal initialization. For reweighted loss, we further compute the products $\lambda_i \sigma_i$ of the effective weights $\lambda_i$ in Eq. \eqref{eq:lambdab def2} and the diagonal entries $\sigma_i$ of $\Sigmab$ in \eqref{eq:Sigmab}.
\end{proof}

\begin{figure*}[t]
   \centering
   \includegraphics[width=\textwidth]{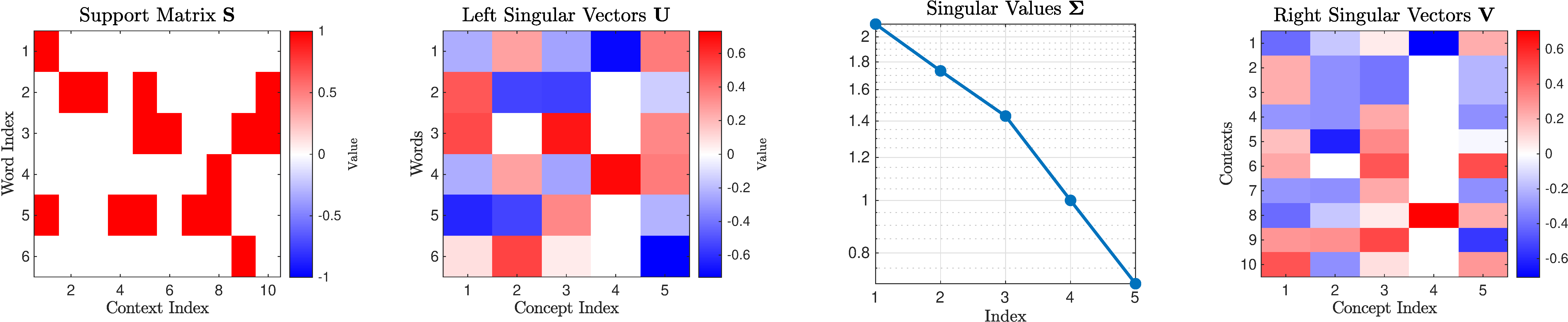}
   \vspace{1em}
   \hspace{-0.38in}\includegraphics[width=1.1\textwidth]{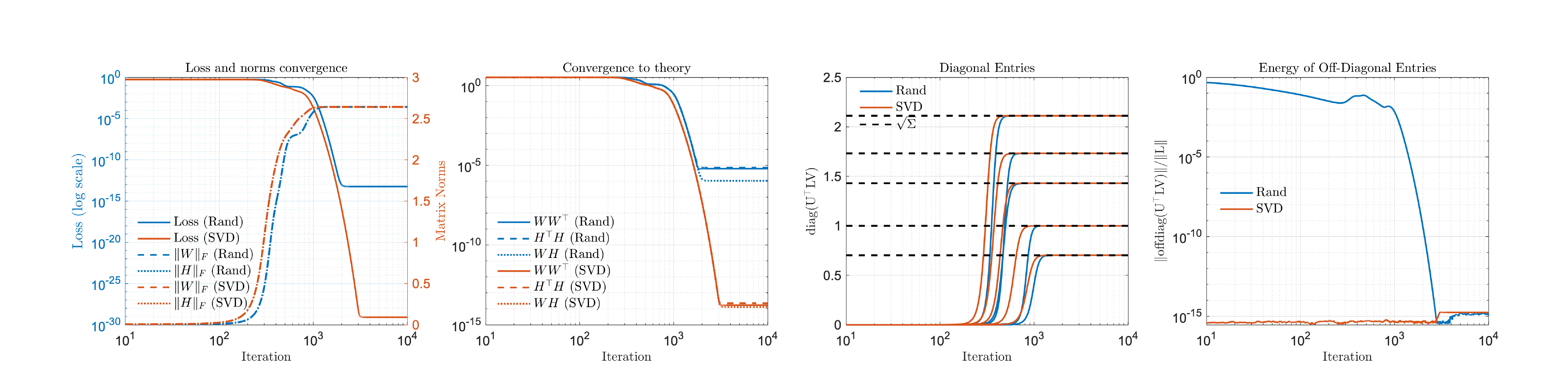}
\captionsetup{width=\textwidth}
   \caption{  \textbf{(Top)} Support matrix and SVD factors of centered support matrix for a synthetic example. \textbf{(Bottom)} Training dynamics of GD minimization of UFM with square loss (Eq. \eqref{eq:UFM_L2} for two initializations: (i) SVD: initialize $\W$ and $\Hb$ as per Thm. \ref{thm:saxe} for $\delta=8$. (ii) Rand: intialize $\W$ and $\Hb$ random Gaussian scaled to match the norm of SVD initialization. Dynamics with the two initialization are shown in orange (SVD) and blue (Rand), respectively. Qualitatively the behavior is similar. \emph{Left:} Training loss and norms of paramteres. \emph{Middle-Left:} Convergence of word and context gram-matrices and of logits to the theory predicted by Thm. \ref{thm:saxe}. \emph{Middle-Right:} Convergence of singular values of logit matrix to those of $\Sigmab$ (see Thm. \ref{thm:saxe}. \emph{Right:} Projection of logits to subspace orthogonal to $\Ub$ and $\Vb$; Logits with Rand initialization initially have non-zero projection but it becomes zero as training progresses.}
   \label{fig:convergence}
\end{figure*}

\end{document}